\documentclass[11pt]{article} 
\usepackage[lmargin=1in,rmargin=1in, bmargin=1in, tmargin=1in]{geometry}

\usepackage{amsmath,amssymb,amsfonts}
\usepackage{tikz}
\usepackage{hyperref}
\usepackage{color}
\usepackage{xcolor}
\usepackage{mathtools}
\usepackage{algpseudocode}
 \usepackage{algorithm}
\usepackage[normalem]{ulem} 
\usepackage[round]{natbib}
\usepackage{amsthm}

\newcommand{\nextsubmission}[1]{}

\newif\ifcomments
\commentstrue 

\newcommand{\given}{\;|\;}
\newcommand{\N}{\mathbb{N}}
\newcommand{\naturals}{\N}

\newcommand{\eval}[2]{\underset{{#1}}{\mathbb{E}}\left[#2\right]}
\newcommand{\prob}[2]{\underset{{#1}}{\Pr}\left(#2\right)}

\newcommand{\A}{\mathcal{A}}

\newcommand{\G}{\mathcal{G}}
\renewcommand{\H}{\mathcal{H}}

\newcommand{\X}{\mathcal{X}}
\newcommand{\Y}{\mathcal{Y}}

\newcommand{\R}{\mathsf{R}}
\newcommand{\VC}{\mathsf{VC}}
\newcommand{\DS}{\mathsf{DS}}
\newcommand{\Ndim}{\mathsf{N}}
\newcommand{\Gdim}{\mathsf{G}}
\newcommand{\psidim}{\Psi\text{-}\mathrm{dim}}

\newcommand{\cNdim}{c\text{-}\Ndim}
\newcommand{\cGdim}{c\text{-}\Gdim}

\newcommand{\cpsidim}{c\text{-}\psidim}

\newtheorem{theorem}{Theorem}
\newtheorem{proposition}[theorem]{Proposition}
\newtheorem{lemma}[theorem]{Lemma}
\newtheorem{corollary}[theorem]{Corollary}

\newtheorem{definition}[theorem]{Definition}

\newtheorem{observation}[theorem]{Observation}
\newtheorem{remark}[theorem]{Remark}

\newtheorem{fact}[theorem]{Fact}

\title{On the Computability of Multiclass PAC Learning}

\usepackage{authblk}
\author{Pascale Gourdeau\thanks{Vector Institute \& University of Toronto, Toronto, ON, Canada; \texttt{pascale.gourdeau@vectorinstitute.ai}} ,  Tosca Lechner\thanks{Vector Institute \& University of Toronto, Toronto, ON, Canada; \texttt{tosca.lechner@vectorinstitute.ai}} ,  and Ruth Urner\thanks{Lassonde School of Engineering, EECS Department, York University, Toronto, ON, Canada; \texttt{ruth@eecs.yorku.ca}}}
\date{\vspace{-5ex}}
\begin{document}

\maketitle

\begin{abstract}%
    We study the problem of \emph{computable} multiclass learnability within the Probably Approximately Correct (PAC) learning framework of \cite{valiant1984theory}.
    In the recently introduced computable PAC (CPAC) learning framework of \cite{agarwal2020learnability}, both learners and the functions they output are required to be computable.
    We focus on the case of finite label space and start by proposing a computable version of the Natarajan dimension and showing that it characterizes CPAC learnability in this setting.
    We further generalize this result by establishing a meta-characterization of CPAC learnability for a certain family of dimensions: \emph{computable distinguishers}.
    Distinguishers were defined by \cite{ben1992characterizations} as a certain family of embeddings of the label space, with each embedding giving rise to a dimension. It was shown that the finiteness of each such dimension characterizes multiclass PAC learnability for finite label space in the non-computable setting. We show that the corresponding computable dimensions for distinguishers characterize CPAC learning.
    We conclude our analysis by proving that the DS dimension, which characterizes PAC learnability for infinite label space, cannot be expressed as a distinguisher (even in the case of finite label space).
\end{abstract}

\section{Introduction}

One of the fundamental lines of inquiry in learning theory is to determine when learning is possible (and how to learn in such cases).
Taking aside \emph{computational efficiency}, which requires learners to run in polynomial time as a function of the learning parameters, the vast majority of the learning theory literature has introduced characterizations of learnability in a purely \emph{information-theoretic} sense, qualifying or quantifying the amount of \emph{data} needed to  guarantee (or refute) generalization, without computational considerations.
These works consider learners as \emph{functions} rather than algorithms.
Recently, \cite{agarwal2020learnability} introduced the Computable Probably Approximately Correct (CPAC) learning framework, adding computability requirements to the pioneering Probably Approximately Correct (PAC) framework of \cite{valiant1984theory}, introduced for binary classification.
In the CPAC setting, both learners and the functions they output are required to be computable functions.
Perhaps surprisingly, adding such requirements substantially changes the learnability landscape.
For example, in a departure from the standard PAC setting for binary classification, where a hypothesis class is learnable in the agnostic case whenever it is learnable in the realizable case (and in particular, always through empirical risk minimization (ERM)), there are binary classes that are CPAC learnable in the realizable setting, but not in the agnostic one. 
For the latter, it is in fact the \emph{effective VC dimension} that characterizes CPAC learnability  \citep{sterkenburg2022characterizations,delle2023find}.
While the works of \cite{agarwal2020learnability,sterkenburg2022characterizations,delle2023find} together provide a practically complete picture of CPAC learnability for the binary case, delineating the CPAC learnability landscape in the \emph{multiclass} setting, where the label space is larger than two and can in general be infinite, has not yet been attempted. 

Multiclass PAC learning has been found to exhibit  behaviours that depart from the binary case even without computability requirements.
For example, not all ERMs are equally successful from a sample-complexity stem point, and, in particular, for infinite label space, the ERM principle can fail \citep{daniely2011multiclass}!
For the finite label space, the finiteness of both the Natarajan and graph dimensions characterizes learnability \citep{natarajan1988two,natarajan1989learning}.
As shown by \cite{ben1992characterizations}, it is actually possible to provide a meta-characterization through \emph{distinguishers}, which are families of functions that map the label space to the set $\{0,1,*\}$.
In the infinite case, both the finiteness of the Natarajan and graph dimensions fail to provide a characterization, and it is the finiteness of the DS dimension that is equivalent to learnability \citep{daniely2014optimal,brukhim2022characterization}.

Ultimately, both the CPAC and multiclass settings  exhibit a significant contrast with PAC learning for binary classification.
In this work, we investigate these two settings in conjunction, and thus initiate the study of computable multiclass PAC learnability. 
We focus on the finite label space, and provide a \emph{meta-characterization} of learnability in the agnostic case: the finiteness of the \emph{computable} dimensions of distinguishers, which we introduce in this work, is both necessary and sufficient for CPAC learnability here.
We also explicitly derive the result in the case of the \emph{computable} Natarajan dimension, also defined in this work, as the lower bound for distinguishers utilizes a lower bound in the computable Natarajan dimension.
The significance of the meta-characterization is two-fold: first, it establishes that computable versions of other known dimensions, namely those that can be defined through a suitable embedding of the label space, also provide a characterization of multi-class CPAC learnability (this applies, for example, to the \emph{computable} graph dimension);  further, it allows us to extract high-level concepts and proof mechanics regarding computable dimensions, which may be of independent interest.
We conclude our investigations by proving that the DS dimension cannot be expressed through the framework of distinguishers, opening the door to potentially more complex phenomena in the infinite label set case.

\subsection{Related Work}

\paragraph{Computable Learnability.}
Following the work of \cite{ben2019learning}, who showed that the learnability of certain basic learning problems is undecidable within ZFC, \cite{agarwal2020learnability} formally integrated the notion of computability within the standard PAC learning framework of \cite{valiant1984theory}.
This novel set-up,  called \emph{computable} PAC (CPAC) learning, requires that both learners and the hypotheses they output be computable.
With follow-up works by \cite{sterkenburg2022characterizations} and \cite{delle2023find}, CPAC learnability in the binary classification setting was shown to be fully characterized by the finiteness of the \emph{effective} VC dimension, formally defined by \cite{delle2023find}.
Computability has since been studied in the context of different learning problems: \cite{ackerman2022computable} extended CPAC learning to continuous domains, \cite{hasrati2023computable} and \cite{delle2024effective} to online learning, and  \cite{gourdeau2024computability} to adversarially robust learning.

\paragraph{Multiclass Learnability.}
While the study of computable learnability is a very recent field of study, multiclass learnability has been the subject of extensive research efforts in the past decades.
In the binary classification setting, the VC dimension characterizes PAC learnability \citep{vapnik1971uniform,ehrenfeucht1989general,blumer1989learnability}.
However, the landscape of learnability in the multiclass setting is much more complex.
The PAC framework was extended to the multiclass setting in the works of \cite{natarajan1988two} and \cite{natarajan1989learning}, which gave a lower bound with respect to the Natarajan dimension, and an upper with the graph dimension.
Later, \cite{ben1992characterizations} generalized the notion of dimension for multiclass learning, and provided a meta-characterization of learnability: (only) dimensions that are \emph{distinguishers} characterize learnability in the finite label space setting, with distinguishers encompassing both the Natarajan and graph dimensions.
\cite{haussler1995generalization} later generalized the Sauer-Shelah-Perles Lemma for these families of functions.
\cite{daniely2015multiclass}, originally \citep{daniely2011multiclass}, identified ``good'' and ``bad'' ERMs with vastly different sample complexities, which, in the case of infinite label space, leads to learning scenarios where the ERM principle fails.
\cite{daniely2014optimal} introduced a new dimension, the DS dimension, which they proved was a necessary condition for learnability. 
The breakthrough work of \cite{brukhim2022characterization} closed the problem by showing it is also sufficient for arbitrary label space.
Different multiclass methods and learning paradigms have moreover been explored by \cite{daniely2012multiclass,rubinstein2006shifting,daniely2015inapproximability}.
Finally, multiclass PAC learning has also been studied in relation to boosting \citep{brukhim2021multiclass,brukhim2023improper,brukhim2024multiclass}, universal learning rates \citep{kalavasis2022multiclass,hanneke2023universal}, sample compression  \citep{pabbaraju2024multiclass} and regularization \citep{asilis2024regularization}.

\section{Problem Set-up}

\paragraph{Notation.}
We denote by $\N$ the natural numbers.
For $n\in\N$, let $[n]$ denote the set $\{1,\dots, n\}$.
Given a finite alphabet $\Sigma$, we denote by $\Sigma^*$ the set of all finite words (strings) over $\Sigma$.
For a given set $X=\{x_1,\dots,x_n\}$, we denote by $X_{-i}:=X\setminus\{x_i\}$ the set resulting in removing the element with index $i$.
We will always use the symbol $\subset$ (vs $\subseteq$) to mean \emph{proper} inclusion.
Let $\Y$ be an arbitrary label set containing $0$. 
For a function $h:\N\rightarrow\Y$, denote by  $M(h):=\arg\max_{n\in\N}\{ h(n)\neq0 \}$ the largest natural number that is not mapped to 0 by $h$, with $M(h)=\infty$ if no such $n$ exists.

\paragraph{Learnability.}
Let $\X$ be the input space and $\Y$ the label space. 
We denote by $\H$ a \emph{hypothesis class} on $\X$: $\H\subseteq\Y^\X$.  A learner $\mathcal{A}: \bigcup_{i=1}^{\infty}(\X\times\Y)^i \to \Y^{\X}$ is a mapping from finite samples $S=((x_1,y_1),\dots,(x_m,y_m))$ to a function $f:\X \to \Y$. 
Given a distribution $D$ on $\X\times\Y$ and a hypothesis $h\in\H$, the \emph{risk} of $h$ on $D$ is defined as $\R_D(h):=\prob{(x,y)\sim D}{h(x)\neq y}.$
The \emph{empirical risk} of $h$ on a sample $S=\{(x_i,y_i)\}_{i=1}^m\in(\X\times \Y)^m$ is defined as 
$\widehat{\R}_S(h):=\frac{1}{m}\sum_{i=1}^m\mathbf{1}[h(x_i)\neq y_i].$
An \emph{empirical risk minimizer (ERM)} for $\H$,  denoted by $\mathrm{ERM}_{\H}$, is a learner that for an input sample $S$ outputs a function $h'\in{\arg\min}_{h\in\H}\;\widehat{\R}_S(h)$.  

We will focus on the case $\X=\N$, i.e., where the domain is countable.
We work in the \emph{multiclass classification} setting, and thus let $\Y$ be arbitrary.
Throughout the paper, whether we are working in the case $|\Y|<\infty$ or $|\Y|=\infty$ will be made explicit.
The case $|\Y|=2$ is the \emph{binary classification} setting for which the \emph{probably approximately correct} (PAC) framework of \cite{valiant1984theory} was originally defined, though it can straightforwardly be extended to arbitrary label spaces $\Y$:

\begin{definition}[Agnostic PAC learnability] 
    A hypothesis class $\H$ is \emph{PAC learnable in the agnostic setting} if there exists a learner $\A$ and function $m(\cdot,\cdot)$ such that for all $\epsilon,\delta\in(0,1)$ and  for any distribution $D$ on $\X\times\Y$, if the input to $\A$ is an i.i.d. sample $S$ from $D$ of size at least $m(\epsilon,\delta)$, then, with probability at least $(1-\delta)$ over the samples, the learner outputs a hypothesis $\A(S)$ with 
    $\R_D(\A(S))\leq \underset{h\in\H}{\inf}\R_D(h)+\epsilon\enspace.$
    The class is said to be \emph{PAC learnable in the realizable setting} if the above holds under the condition that $\underset{h\in\H}{\inf}\R_D(h)=0$.
\end{definition}

\begin{definition}[Proper vs improper learning]
    Given a hypothesis class $\H\subseteq\Y^\X$, a learner $\A$ is said to be \emph{proper} if for all $m\in\N$ and samples $S\in(\X\times\Y)^{m}$, $\A(S)\in \H$, and \emph{improper} otherwise.
\end{definition}

We note that by definition ERMs are proper learners.

\paragraph{Computable learnability.}
We start with some computability basics.
A function $f:\Sigma^*\rightarrow\Sigma^*$ is called \emph{total computable} if there exists a program $P$ such that, for all inputs $\sigma\in\Sigma^*$, $P$ halts  and satisfies $P(\sigma)=f(\sigma)$. 
A set $S\subseteq\Sigma^*$ is said to be decidable (or recursive) if there exists a program $P$ such that, for all $\sigma\in\Sigma^*$, $P(\sigma)$ halts and outputs whether $\sigma\in S$; $S$ is said to be semi-decidable (or recursively enumerable) if there exists a program $P$ such that $P(\sigma)$ halts for all $\sigma\in S$ and, whenever $P$ halts, it correctly outputs whether $\sigma\in S$.
An equivalent formulation of semi-decidability for $S$ is the existence of a program $P$ that enumerates all the elements of $S$.

When studying CPAC learnability, we consider hypotheses with a mild requirement on their \emph{representation} (note that otherwise negative results are trivialized, as argued by \cite{agarwal2020learnability}):

\begin{definition}[Computable Representation \citep{agarwal2020learnability}]
    A hypothesis class $\H\subseteq\Y^\X$ is called \emph{decidably representable (DR)} if there exists a decidable set of programs $\mathcal P$ such that the set of all functions computed by programs in $\mathcal P$ equals $\H$. 
    The class $\H$ is called \emph{recursively enumerably representable (RER)} if there exists such a set of programs that is recursively enumerable.
\end{definition}

Recall that PAC learnability only takes into account the sample size needed to guarantee generalization. 
It essentially views learners as \emph{functions}.
Computable learnability adds the basic requirement that learners be \emph{algorithms} that halt on all inputs and output total computable functions.

\begin{definition}[CPAC Learnability \citep{agarwal2020learnability}]
     A class $\H\subseteq\Y^\X$ is (agnostic) \emph{CPAC learnable} if there exists a computable (agnostic) PAC learner for $\H$ that outputs  total computable functions as predictors and uses a decidable (recursively enumerable) representation for these.
\end{definition}

\paragraph{Dimensions characterizing learnability.}
A notion of dimension can provide a characterization of learnability for a learning problem in two different senses: first, in a \emph{qualitative} sense, where the finiteness of the dimension is both a necessary and sufficient condition for learnability, and, second, in a \emph{quantitative} sense, where the dimension explicitly appears in both lower and upper bounds on the sample complexity. 
See \cite{pmlr-v247-lechner24a} for a thorough treatment of dimensions in the context of learnability.

In the case of binary classification, the Vapnik-Chervonenkis (VC) dimension characterizes learnability in a quantitative sense (of course implying a qualitative characterization as well):

\begin{definition}[VC dimension \citep{vapnik1971uniform}]
    Given a class of functions $\H$ from $\X$ to $\{0,1\}$, we say that a set $S\subseteq\X$ is \emph{shattered by $\H$} if the restriction of $\H$ to $S$ is the set of all function from $S$ to $\{0,1\}$.
    The VC dimension of a hypothesis class $\mathcal{H}$, denoted $\VC(\mathcal{H})$, is the size $d$ of the largest set that can be shattered by $\mathcal{H}$.
    If no such $d$ exists then $\VC(\mathcal{H})=\infty$.
\end{definition}

In the multiclass setting, the Natarajan dimension provides a lower bound on learnability, while the graph dimension, an upper bound \citep{natarajan1988two,natarajan1989learning}.
When the label space is  \emph{finite}, both dimensions characterize learnability, though they can be separated by a factor of $\log(|\Y|)$ \citep{ben1992characterizations,daniely2011multiclass}. 

\begin{definition}[Natarajan dimension \citep{natarajan1989learning}]
\label{def:natarajan}
    A set $S=\{x_1,\dots,x_n\}\in\X^k$ is said to be \emph{N-shattered} by $\H$ if there exists labelings $g_1,g_2\in\Y^k$ such that for all $i\in [k]$, $g_1(i)\neq g_2(i)$ and for all subsets $I\subseteq [k]$ there exists $h\in\H$ with $h(x_i)=g_1(i)$ in case $i\in I$ and $h(x_i)=g_2(i)$ in case $i\in [k]\setminus I$.
    The \emph{Natarajan dimension} of $\H$, denoted $\Ndim(\H)$, is the size $d$ of the largest set that can be N-shattered by $\mathcal{H}$.
    If no such $d$ exists then $\Ndim(\mathcal{H})=\infty$.
\end{definition}

\begin{definition}[Graph dimension \citep{natarajan1989learning}]
    A set $S=\{x_1,\dots,x_n\}\in\X^k$ is said to be \emph{G-shattered} by $\H$ if there exists a labeling $f\in\Y^k$ such that for every $I\subseteq [k]$ there exists $h\in\H$ such that for all $i\in I$, $h(x_i)=f(i)$ and for all $i\in [k]\setminus I$, $h(x_i)\neq f(i)$.
    The \emph{graph dimension} of $\H$, denoted $\Gdim(\H)$, is the size $d$ of the largest set that can be G-shattered by $\mathcal{H}$.
    If no such $d$ exists then $\Gdim(\mathcal{H})=\infty$.
\end{definition}

When $\Y$ is infinite, it is the DS dimension that characterizes learnability \citep{daniely2014optimal,brukhim2022characterization}.
Before defining it, we need to define pseudo-cubes.

\begin{definition}[Pseudo-cube]
    A set $H\subseteq\Y^d$ is called a \emph{pseudo-cube of dimension $d$} if $H$ is non-empty and finite, and for every $h\in H$ and every index $i\in[d]$ there exists $g\in H$ such that $h(j)=g(j)$ if and only if $j\neq i$. 
\end{definition}

\begin{definition}[DS dimension \citep{daniely2014optimal}]
    A set $S\in\X^n$ is said to be DS-shattered by $\H\subseteq \Y^\X$ if $\H|_S$ contains an $n$-dimensional pseudo-cube. 
    The \emph{DS dimension} of $\H$, denoted $\DS(\H)$, is the size $d$ of the largest set that can be DS-shattered by $\mathcal{H}$.
    If no such $d$ exists then $\DS(\mathcal{H})=\infty$.
\end{definition}

We refer the reader to \cite{brukhim2022characterization} for results separating the Natarajan and DS dimensions, as well as an example showing why the finiteness of the pseudo-cube is a necessary property in order for the DS dimension to characterize learnability.

\section{Characterizing CPAC Learnability with the Computable Natarajan Dimension}
\label{sec:c-nat-graph}

In this section, we first recall results in the binary CPAC setting that have implications for multiclass CPAC learning, and include  conditions under which CPAC learnability is sufficient. We then define the computable versions of the Natarajan and graph dimensions, in the spirit of the effective VC dimension, implicitly appearing in the work of \cite{sterkenburg2022characterizations} and formally defined by \cite{delle2023find}.
We also show that the same gap as in the standard (non-computable) setting exists between the computable Natarajan and computable graph dimension. 
In Section~\ref{sec:c-ndim-lb}, we show that the finiteness of the computable Natarajan dimension is a necessary condition for multiclass CPAC learnability for arbitrary label spaces.
We finish this section by showing that this finitess is sufficient for \emph{finite} label spaces in Section~\ref{sec:c-ndim-ub-finite}.

We note that there are several hardness results for binary CPAC learning that immediately imply hardness for multiclass CPAC learning. In particular, the results that show a separation between agnostic PAC and CPAC learnability (both for proper \citep{agarwal2020learnability} and improper \citep{sterkenburg2022characterizations} learning) imply that there are decidably representable (DR) classes which are (information-theoretically) multiclass learnable, but which are not computably multiclass learnable. 
On the other hand, in the binary case, any PAC learnable class that is recursively enumerably representable (RER) is also CPAC learbable in the \emph{realizable case}. 
For multiclass learning, we can similarly implement an ERM rule for the realizable case as outlined below. 
\begin{proposition}
\label{prop:rer-sufficient-cpac-realizable}
    Let $\H$ be RER. 
    If $\Gdim(\H)<\infty$ or $\Ndim(\H)\log(|\Y|)<\infty$, then $\H$ is properly CPAC learnable in the realizable setting.
\end{proposition}
    
\begin{proof}
    The conditions $\Gdim(\H)<\infty$ and $\Ndim(\H)\log(|\Y|)<\infty$ are both sufficient to guarantee generalization with an ERM.
    Upon drawing a sufficiently large sample $S$ from underlying distribution $D$, it suffices to enumerate all $h\in\H$ and compute $\widehat{\R}_S(h)$ one by one until we obtain one with zero empirical risk. 
    We thus have a computable ERM (recall that all $h\in\H$ are total computable). 
\end{proof}

Thus, RER classes that satisfy the uniform convergence property can be CPAC learned in the realizable case. 
Moreover, having access to a computable ERM yields the following: 

\begin{fact}
    \label{fact:c-erm+finite-Gdim-Ndim=cpac}
    Let $\H$ have a computable ERM and suppose $\Gdim(\H)<\infty$ or $\Ndim(\H)\log(|\Y|)<\infty$. 
    Then $\H$ is CPAC learnable in the multiclass setting.
\end{fact}

\paragraph{The Computable Natarajan and Graph Dimensions.}
The general idea in defining \emph{computable} versions of shattering-based dimensions, such as the effective VC dimension \citep{sterkenburg2022characterizations,delle2023find} and the computable robust shattering dimension \citep{gourdeau2024computability}, is to have a computable proof of the statement ``$X$ cannot be shattered'' for all sets of a certain size. 
We define the computable  Natarajan and graph dimensions in this spirit as well:

\begin{definition}[Computable Natarajan dimension]
\label{def:c-natarajan}
A \emph{$k$-witness of Natarajan dimension} for a hypothesis class $\H$ is a function $w_{N}:\X^{k+1}\times\Y^{k+1}\times\Y^{k+1}\rightarrow 2^{k+1}$ that takes as input a set $S=\{x_1,\dots,x_{k+1}\}$ of size $k+1$ and two labelings $g_1,g_2\in\Y^{k+1}$ of $S$ satisfying $g_1(i)\neq g_2(i)$ for all $i\in [k+1]$ and outputs a subset $I\subseteq [k+1]$ such that for every $h\in\H$ there exists $i\in [k+1]$ such that $h(x_i)\neq g_1(i)$ if $i\in I$ and $h(x_i)\neq g_2(i)$ if $i\in [k+1]\setminus I$.
The \emph{computable Natarajan dimension} of $\H$, denoted $\cNdim(\H)$, is the smallest integer $k$ such that there exists a \emph{computable} $k$-witness of  Natarajan dimension for~$\H$. \end{definition}

\begin{remark}
    In the definition above, we make explicit the requirement that $g_1$ and $g_2$ differ on all indices, but this can be checked computably. 
    Moreover, the usual manner to obtain computable dimensions is to negate the first-order formula for ``$X$ is shattered by $\H$''. 
    In the Natarajan dimension case would give a witness that after finding $I$, whenever given a hypothesis $h$, outputs the $x\in S$ satisfying the condition, but it is straightforward to computably find once $I$ is obtained. 
    We thus simplify the Natarajan, graph and general dimensions (see Section~\ref{sec:gen-method}) in this manner.
\end{remark}

\begin{definition}[Computable graph dimension]
    A \emph{$k$-witness of graph dimension} for a hypothesis class $\H$ is a function $w_{G}:\X^{k+1}\times\Y^{k+1}\rightarrow 2^{k+1}$ that takes as input a set $S=\{x_1,\dots,x_{k+1}\}$ of size $k+1$ and a labeling $f\in\Y^{k+1}$ of $S$ and outputs a subset $I\subseteq [k+1]$ such that for every $h\in\H$ there exists $i\in [k+1]$ such that $h(x_i)\neq f(i)$ if $i\in I$ and $h(x_i)= f(i)$ if $i\in [k+1]\setminus I$.
    The \emph{computable graph dimension} of $\H$, denoted $\cGdim(\H)$, is the smallest integer $k$ such that there exists a \emph{computable} $k$-witness of  graph dimension for~$\H$.
\end{definition}

In the binary setting, the VC, Natarajan,  and graph dimensions are all identical.
This is also the case for the computable counterparts. In particular, this implies that $\cNdim(\H)$ and $\Ndim(\H)$ can be arbitrarily far apart, with $\Ndim(\H)=1$ and $\cNdim(\H) =\infty $. The same separation holds for the graph dimension and computable graph dimension.
It is also straightforward to check that $\cNdim(\H)\leq\cGdim(\H)$.
Now, as in the non-computable versions of the Natarajan and graph dimensions, we have an arbitrary gap between their computable counterparts (see Appendix~\ref{appx:cndim-cgdim-gap} for the proof):

\begin{proposition}
\label{prop:cndim-cgdim-gap}
    For any $m\in\N\cup\{\infty\}$ there exist $\X,\Y$ and  $\H$ with $|\X|=m$, $\cNdim(\H)=1$ and $\cGdim(\H)=m$.
\end{proposition}

\subsection{The Finiteness of the Computable Natarajan Dimension as a Necessary Condition}
\label{sec:c-ndim-lb}

In this section, we show that the finiteness of the computable Natarajan dimension is a necessary condition for CPAC learnability in the agnostic setting, even in the case $|\Y|=\infty$. 

\begin{theorem}
\label{thm:m-cpac-lb-ndim}
    Let $\H\subseteq\Y^\X$ be improperly CPAC learnable.
    Then $\cNdim(\H)<\infty$, i.e. $\H$ admits a computable $k$-witness of Natarajan dimension for some $k\in\N$.
\end{theorem}

We will first show a multiclass analogue of the computable No-Free-Lunch theorem for binary classification (Lemma 19 in \citep{agarwal2020learnability}), adapted with the Natarajan dimension in mind:

\begin{lemma}
\label{lemma:cmnflt}
    For any computable learner $\A$, for any $m\in\N$, any instance space $\X$ of size at least $2m$, any subset $X=\{x_1,\dots,x_{2m}\}$ of size at least $2m$, and any two functions $g_1,g_2:X\rightarrow\Y$ satisfying $g_1(x)\neq g_2(x)$ for all $x\in X$, we can computably find $f:X\rightarrow\Y$ such that 
    \begin{enumerate}
        \item $f(x)\in\{g_1(x),g_2(x)\}$ for all $x\in X$,
        \item $\R_{D}(f)=0$,
        \item With probability at least $1/7$ over $S\sim D^m$, $\R_D(\A(S))\geq 1/8$,
    \end{enumerate}
    where $D$ is the uniform distribution on $\{(x_i, f(x_i))\}_{i=1}^{2m}$.
\end{lemma}

\begin{proof}[Proof sketch]
    We will first prove the existence of a pair $(f,D)$ satisfying the desired requirements.
    To this end, for $I\subseteq [2m]$, denote by $f_I:X\rightarrow\Y$ the labelling of $X$ satisfying $f_I(x_i)=g_1(x_i)$ if $i\in I$ and $f_I(x_i)=g_2(x_i)$ if $i\in [2m]\setminus I$. 
    For each $f_I$, define the following distribution $D_I$ on $X\times \Y$:
    \begin{equation*}
        D_I((x,y))
        =\begin{cases}
            \frac{1}{2m}   & \text{if } y=f_I(x)\\
            0               &\text{otherwise}
        \end{cases}
        \enspace.
    \end{equation*}
    Note that there are $T=2^{2m}$ possible such functions from $X$ to $\Y$.
    Let $\{(f_i,D_i)\}_{i\in[T]}$ denote the set of all such function-distribution pairs and note that $\R_{D_i}(f_i)=0$ for all $i\in[T]$.

    Note that, by a simple application of Markov's inequality, it is sufficient to show that there exists $i$ satisfying
    \begin{equation}
    \label{eqn:mnflt-exp-lb}
        \eval{S\sim D^m_i}{\R_{D_i}(\A(S))} \geq 1/4\enspace.
    \end{equation}
    The proof that the third requirement is satisfied is nearly identical to that of the No-Free-Lunch theorem (see, e.g., \cite{shalev2014understanding}, Theorem 5.1), and is omitted for brevity. 
    
    Now, to computably find a pair $(f,D)$ satisfying $\eval{S\sim D^m_i}{\R_{D_i}(\A(S))} \geq \frac{1}{4}$, 
    it suffices to note that $\A$ is computable (and outputs computably evaluable functions), that the set $\{(f_i,D_i)\}_{i\in[T]}$ is finite, and that for each pair $(f_i,D_i)$, we can use $\A$ to compute the expected risk.
    Indeed, denote by $S_1,\dots,S_n$ the $n=(2m)^m$ possible sequences of length $m$ from $X$, and for some $i\in[T]$ and $S_j=(x_1,\dots,x_m)$, let $S_j^i:=((x_l,f_i(x_l)))_{l=1}^m$ be the sequence $S_j$ labeled by $f_i$.
    Each distribution $D_i$ induces the equally likely sequences $S_1^i,\dots,S_n^i$, implying
    \begin{equation*}
        \eval{S\sim D^m_i}{\R_{D_i}(\A(S))} = \frac{1}{n}\sum_{j=1}^n \R_{D_i}(\A(S_j^i))
        = \frac{1}{n}\sum_{j=1}^n \frac{1}{2m}\sum_{l=1}^{2m}\mathbf{1}[\A(S_j^i)(x)\neq f_i(x)]\enspace.
    \end{equation*}
    Since such a pair must exist, we will eventually stop for some $i$ for which Equation~\ref{eqn:mnflt-exp-lb} holds.
\end{proof}

We are now ready to prove Theorem~\ref{thm:m-cpac-lb-ndim}, in the spirit of the binary case \citep{sterkenburg2022characterizations}.

\begin{proof}[Proof of Theorem~\ref{thm:m-cpac-lb-ndim}]
    Let $\A$ be a computable (potentially improper) learner for $\H$ with sample complexity function $m(\epsilon,\delta)$.
    Let $m=m(1/8,1/7)$.
    We will show that $\A$ can be used to build a computable $(2m-1)$-witness of Natarajan dimension for $\H$.
    
    To this end, suppose we are given an arbitrary set  $X=\{x_1,\dots,x_{2m}\}\in\X^{2m}$ and labelings $g_1,g_2:X\rightarrow\Y$  satisfying $g_1(x_i)\neq g_2(x_i)$ for all $i\in [2m]$.
    By Lemma~\ref{lemma:cmnflt}, we can computably find $f:X\rightarrow\Y$ such that (i) $f(x)\in\{g_1(x),g_2(x)\}$ for all $x\in X$, (ii) $\R_{D}(f)=0$, and (iii) $\prob{S\sim D^m}{\R_D(\A(S))\geq 1/8}\geq 1/7$,
    where $D$ is the uniform distribution on $\{(x_i, f(x_i))\}_{i=1}^{2m}$.
    This implies that the labeling of $X$ induced by $f$ is not achievable by any $h\in\H$: otherwise $\underset{h\in\H}{\min}\;\R_D(h)=0$, and by the PAC guarantee $ \prob{S\sim D^m}{\R_D(\A(S))\geq \underset{h\in\H}{\min}\;\R_D(h) + 1/8}< 1/7$,
    we would get $\prob{S\sim D^m}{\R_D(\A(S))\geq 1/8}< 1/7$,
    a contradiction.
    Now, let $I\subseteq[2m]$ be the index set identifying the instances in $X$ labelled by $g_1$ in~$f$.
    Then clearly $I$ is the set that we seek: for every $h\in\H$ there exists $x_i\in X$ such that $h(x_i)\neq f(x_i) $, where $f(x_i)=g_1(x_i)$ if $ i\in I$ and $g_2(x_i)$ if  $i\in [2m]\setminus I$,
    as required.
\end{proof}

\subsection{The Finiteness of the Computable Natarajan Dimension as a Sufficient Condition}
\label{sec:c-ndim-ub-finite}

We now state and show the main result of the section: finite computable Natarajan dimension is sufficient for CPAC learnability whenever $\Y$ is finite.

\begin{theorem}
\label{thm:c-nat-sufficient}
    Let $\cNdim(\H)<\infty$ and $|\Y|<\infty$. Then $\H$ is (improperly) CPAC  learnable.
\end{theorem}

In the binary classification setting, \cite{delle2023find} showed that finite effective VC dimension is sufficient for CPAC learnability. 
We generalize this approach to the multiclass setting:

\begin{proof}[Proof of Theorem~\ref{thm:c-nat-sufficient}]
    Let $\H\subseteq\Y^\X$ be such that $\cNdim(\H)=k$.
    Let $w:\X^{k+1}\times\Y^{k+1}\times\Y^{k+1}\rightarrow 2^{k+1}$  be a $k$-witness of Natarajan dimension.
    We will embed $\H$ into $\H'$ satisfying (i) $\Ndim(\H')\leq k+1$ and (ii)~$\H'$ has a computable ERM. 
    By Fact~\ref{fact:c-erm+finite-Gdim-Ndim=cpac}, this is sufficient to guarantee multiclass CPAC learnability.
    Before showing that properties (i) and (ii) hold, we introduce the following notation.
    Given $y,y'\in\Y^{k+1}$ and a subset $I\subseteq[k+1]$, we denote by $f_{I,y,y'}\in\Y^{[k+1]}$ the function 
    \begin{equation}
    \label{eqn:f_I}
        f_{I,y,y'}(i) = 
        \begin{cases}
            y_i & i\in I \\
            y_i' & i\in[k+1]\setminus I
        \end{cases}
        \enspace.
    \end{equation}
    When $y,y'$ are fixed and clear from context, we will shorten the notation to $f_I$ for readability. 
We first show the following lemma, which will be invoked in Section~\ref{sec:gen-method} as well, when we give a more general necessary condition on multiclass CPAC learnability (Theorem~\ref{thm:lb-c-nat-c-dim}).
\begin{lemma}\label{lemma:embedding}
    For every $\H\subseteq\Y^\X$ with $|\Y|<\infty$ and $\cNdim(\H) = k_N$, there exists a class $\H' \supset \H$ with 
    \begin{itemize}
        \item $\Ndim(\H') \leq k_N +1$
        \item There exists a computable function $v: \bigcup_{m=1}^{\infty }\X^m \to \bigcup_{m=1}^{\infty } 2^{\Y^m}$, that takes as input a finite domain subset $T\subset \X$ and outputs a set of labelings $v(T) = \H'|_T$. 
    \end{itemize}
\end{lemma}

\begin{proof}
    \underline{Constructing $\H'$.} 
    Consider the class $\G\subseteq\Y^\N$ of ``good'' functions satisfying, for all $g\in\G$
    \begin{enumerate}
        \item $ M(g)<\infty$, where $M(g)=\arg\max_{n\in\N\cup\{\infty\}}\{ g(n)\neq0 \}$.
        \item For any $x_1<\dots<x_k<x_{k+1}\leq M(g)$, and any labelings $y,y'\in\Y^{k+1}$, let $w(X,y,y')=I\subseteq[k+1]$, where $X=\{x_i\}_{i\in[k+1]}$. Then $g|_X\neq f_{I,y,y'}|_X$.
    \end{enumerate}
   Namely, ``good'' functions defined on $\N$ are those that are eventually always 0 and do not encode the output of the witness function for any labelings.
    
    Now, let $\H':=\H\cup\G$. We will show that $\H'$ indeed satisfies the conditions above.

    \underline{$\Ndim(\H')\leq k+1$.} 
    Let $X=\{x_1,\dots,x_{k+2}\}\in\X^{k+2}$ and $y,y'\in\Y^{k+2}$ be arbitrary labelings that differ in each component, i.e., $y_i\neq y'_i$ for all $i\in[k+2]$.
    WLOG, suppose $x_1<\dots<x_{k+1}<x_{k+2}$ and that $y_{k+2}>y_{k+2}'$, in particular $y_{k+2} >0$.
    Let $J$ be the output of the $k$-witness $w$ on $(X,y,y')$ without the $k+2$-th entries, i.e., $J:=w(X_{-(k+2)},y_{-(k+2)},y'_{-(k+2)})$.
    Let $J^+=J\cup\{k+2\}$, and, by a slight abuse of notation, let $f_J\in\Y^{[k+1]}$ and $f_{J^+}\in\Y^{[k+2]}$, defined as per Equation~\ref{eqn:f_I} and where we omit $y,y'$ in the subscript for readability.
    We claim that there exists no $h\in\H'$ satisfying $h|_X=f_{J^+}|_X$.
    First note that no $h\in\H$ can satisfy this, because $J$ is defined as the output of the $k$-witness $w$.
    Then $h$ must be in $\G$.
    We distinguish two cases:
    \begin{enumerate}
        \item $h(x_{k+2})=0$ : then $h(x_{k+2})\neq y_{k+2}=f_{J^+}(x_{k+2})$,
        \item $h(x_{k+2})\neq 0$ : then $x_{k+2}\leq M(h)$, which by definition implies $h|_{X_{-(k+2)}}\neq f_J$.
    \end{enumerate}
    
\underline{Existence of the computable function $v$.}
    Let $T\subseteq \X^m$ and $S\in(\X\times\Y)^m \subseteq \H'|_{T}$ be arbitrary. We will argue that (i) we can computably obtain all labellings $\G|_T$ and (ii) $\G|_T = \H'|_T$.
    We first argue that we can computably obtain all labelings in $\G|_T$. Let $M=\max_{x\in T}x$. Note that in order to find all labellings in $\G|_T$ it suffices to consider functions $h$ with $ M(h)\leq M$, and that by the finiteness of $\Y$, there are a finite number of ``good'' functions in $\G$ satisfying this. These function can now be computably identified by first listing all patterns $\Y^{m}$ and then using the computable witness function $w_N$ on all inputs 
    \begin{equation}
    \label{eqn:patterns}
        \{(U,y,y'): U\subseteq [M], y,y\in \Y^M \text{ and for all } i \in [M] \text{ we have } y_i \neq y'_i \}
    \end{equation} 
    to exclude those patterns that are not in $\G$, which is possible by the finiteness of $\Y^M$. By definition of $\G$ the remaining patterns match $\G|_T$, thus showing that there is indeed an algorithm that for any $T$ outputs $\G|_T$.
    We now argue that $\H'|_T = \G_T$. 
    Since $\G\subseteq \H'$ it is sufficient to argue that for any labelling $h\in \H'|_T$ we have $h\in \G|_T$.
    Let $h\in \H'$ be arbitrary. Now consider its ``truncated'' version $h_M$, where if $x\leq M$, $h_M(x)=h(x)$ and otherwise, $h_M(x)=0$. 
    We will now show that $h \in \G$.  
    Suppose $h|_{[M]} \notin \G|_{[M]}$
    Then there must exist $X=\{x_1,\dots,x_{k+1}\}\subseteq[M]$ and $y,y'\in\Y^{k+1}$ such that for $I=w_N(X,y,y')$, $h^*_M|_X=f_I|_X$, but that means that $h\in\H$ also satsifies this, a contradiction by the definition of $w_N$. 
    Thus $h|_{[M]} \in \G|_{[M]}$ and since $T\subseteq [M]$, $h|_T \in \G|_{T}$. Therefore $\H'_T = \G_T$, concluding our proof.
\end{proof}
It now remains to show that $\mathrm{ERM}_{\H'}$ is computable to conclude the proof of Theorem~\ref{thm:c-nat-sufficient}. We note that $\G$ is recursively enumerable and thus we can iterate through all elements of the class $\G$. Furthermore we have seen that that since for any sample $S$ we can computably find all behaviours $\H'|_S = \G|_S$ and thus have a stopping criterion for $\mathrm{ERM}_{\G}$ which also serves as an implementation of $\mathrm{ERM}_{\H'}$.
\end{proof}

We also note that the above proof goes through in case $\Y$ is infinite, but the range of possible labels for each initial segment of $\H$ is computably bounded:
\begin{observation}
    Let $\X= \naturals$ and $\Y= \naturals $. Furthermore, let $\H$ be a hypothesis class with $\cNdim(\H)=k$.  If there is a computable function $c: \naturals \to \naturals $, such that for every $n\in \naturals$, $\H|_{[n]} \subseteq [c(n)]^{[n]}$, then $\H$ is agnostically CPAC learnable.
\end{observation}
We note that this condition would capture many infinite-label settings, such as question-answering, with the requirement that the length of the answer be bounded as a function of the length of the question. 
The proof of this observation can be obtained by replacing Equation~\ref{eqn:patterns} by  $\{(U,y,y'): U\subseteq [M], y,y\in [c(M)]^M \text{ and for all } i \in [M] \text{ we have } y_i \neq y'_i \}$ in the construction of $v(T)$ in the proof of Lemma~\ref{lemma:embedding}. 

\section{A General Method for CPAC Learnability in the Multiclass Setting with $|\Y|<\infty$}
\label{sec:gen-method}

The Natarajan dimension is one of many ways to generalize the VC dimension to arbitrary label spaces: the graph and DS dimensions also generalize the VC dimension, the latter characterizing learnability even in the case of infinitely many labels \citep{daniely2014optimal,brukhim2022characterization}.
\cite{ben1992characterizations} generalized a notion of shattering for finite label spaces by encoding the label space into the set $\{0,1,*\}$, which subsumes Natarajan and graph shattering.

In this section, we formalize a new, more general notion of \emph{computable} dimension, which are based on those presented by \cite{ben1992characterizations}.
We show that the finiteness of these computable dimensions characterizes CPAC learnability for finite label space, notably generalizing the results we presented in Section~\ref{sec:c-nat-graph}.
This general view also allows us to extract a more abstract and elegant relationship between computable learnability and computable dimensions.

Let $\Psi$ be a family of functions from $\Y=\{0,\dots,l\}$ to $\{0,1,*\}$.
Given $n\in \N$, $\bar{\psi}:=(\psi_1,\dots,\psi_n)\in\Psi^n$ and a tuple of labels $y\in\Y^n$, denote by $\bar{\psi}(y)$ the tuple $(\psi_1(y_1),\dots,\psi(y_n))$.
Given a set of label sequences $Y\subseteq\Y^n$, we overload $\bar{\psi}$ as follows: $\bar{\psi}(Y):=\{\bar{\psi}(y)\given y \in Y\}$.
We are now ready to define the $\Psi$-dimension.

\begin{definition}[$\Psi$-shattering and $\Psi$-dimension \citep{ben1992characterizations}]
    A set $X\in\X^n$ is \emph{$\Psi$-shattered} by $\H$ if there exists $\bar{\psi}\in\Psi^n$ such that $\{0,1\}^n\subseteq\bar{\psi}(\H|_X)$.
    The \emph{$\Psi$-dimension} of $\H$, denoted $\psidim(\H)$, is the size $d$ of the largest set $X$ that is $\Psi$-shattered by $\H$. 
    If no largest such $d$ exists, then $\psidim(\H)=\infty$.
\end{definition}

Here the condition  $\{0,1\}^n\subseteq\bar{\psi}(\H|_X)$ essentially means that any 0-1 encoding of the labels is captured by applying $\bar{\psi}$  to some $h$ in the projection of $\H$ onto $X$.

\paragraph{Examples.}
The graph dimension corresponds to the $\Psi_\Gdim$-dimension, where $\Psi_\Gdim:=\{\psi_{k}\given k\in\{0,\dots,l\}\}$, where 
$$\psi_{k}(y)
=\begin{cases}
    1 & y=k\\
    0 & \text{otherwise}
\end{cases}\enspace,$$
and the Natarajan dimension to the set $\Psi_{\Ndim}:=\{\psi_{k,k'}\given k\neq k'\in\{0,\dots,l\}\}$, where 
$$\psi_{k,k'}(y)
=\begin{cases}
    1 & y=k\\
    0 & y=k'\\
    * & \text{otherwise}
\end{cases}\enspace.$$

\begin{definition}[Distinguisher \citep{ben1992characterizations}]
    A pair $(y,y')\in\{0,\dots,l\}$ of distinct labels is said to be \emph{$\Psi$-distinguishable} if there exists $\psi\in\Psi$ with $\psi(y)\neq\psi(y')$ and neither $\psi(y)$ nor $\psi(y')$ is equal to $*$.
    The family $\Psi$ is said to be a \emph{distinguisher} if all pairs $(y,y')\in\{0,\dots,l\}$  of distinct labels are $\Psi$-distinguishable.
\end{definition}

The notion of being a \emph{distinguisher} in \cite{ben1992characterizations} was shown to be both necessary and sufficient in order for the $\Psi$-dimension to characterize learnability in the \emph{qualitative} sense, i.e. through its finiteness.
In essence, distinguishers provide a \emph{meta-characterization} of learnability:

\begin{theorem}[Theorem 14 in \citep{ben1992characterizations}]
\label{thm:bd92}
    A family $\Psi$ of functions from $\{0,\dots,l\}$ to $\{0,1,*\}$ provides a characterization of proper learnability if and only if $\Psi$  is a distinguisher.
\end{theorem}

\cite{ben1992characterizations} indeed implicitly define learnability as \emph{proper} learnability. 
Note, however, that the argument showing that being a distinguisher is a necessary condition for characterizing learnability also goes through for \emph{improper} learnability (see Lemma~13 therein).

\paragraph{Computable $\Psi$-Dimensions.}
We can now straightforwardly define $\cpsidim$, as the smallest integer $k\in\N$ for which there exists a computable proof of the statement ``$X$ cannot be $\Psi$-shattered'' for any set $X$ of size larger than $k$:

\begin{definition}[Computable $\Psi$-dimension]
    Let $\H\subseteq \Y^\X$.
    A \emph{$k$-witness} of $\Psi$-dimension is a function $w:\X^{k+1}\times\Psi^{k+1}\rightarrow\{0,1\}^{k+1}$ such that for any sequence $X\in\X^{k+1}$, any $\bar{\psi}\in\Psi^{k+1}$, we have that $w(X,\bar{\psi})\notin\bar{\psi}(\H|_X)$.
    The \emph{computable $\Psi$-dimension} of $\H$, denoted $\cpsidim(\H)$, is the smallest integer $k$ for which there exists a $k$-witness of $\Psi$-dimension.
    If no such $k$ exists, then $\cpsidim(\H)=~\infty$.
\end{definition}

Here, one can view the witness function as returning a 0-1 encoding that no hypothesis in $\H$ projected onto $X$ can achieve when $\bar{\psi}$ is applied to it.

\subsection{Necessary Conditions for Finite $\Y$}
\label{sec:lb-c-psi-dim}

In this section, we show that, for a family $\Psi$ embedding the label space $\Y$ into $\{0,1,*\}$, the finiteness of the computable $\Psi$-dimension is a necessary condition for CPAC learnability of a class $\H$. 
We start by stating a lower bound on the computable Natarajan dimension in terms of the computable $\Psi$ dimension and the size of the label space. 
The theorem is a computable version of Theorem~7 in \cite{ben1992characterizations}.

\begin{theorem}
\label{thm:lb-c-nat-c-dim}
    Let $\Psi$ be a family of functions from $\Y$ to $\{0,1,*\}$.
    For every RER hypothesis class $\H \subseteq \Y^{\X}$ over a finite label space $\Y$, we have that $$ \frac{\cpsidim(\H)}{\log(\cpsidim(\H)) + 2\log(|\Y|)} \leq \cNdim(\H) +1 \enspace.$$
\end{theorem}

Combining Theorem~\ref{thm:lb-c-nat-c-dim} with Theorem~\ref{thm:m-cpac-lb-ndim}, we obtain the following:

\begin{corollary}
\label{cor:c-psi-dim-necessary}
    Let $\Psi$ be a distinguisher, and suppose $\H\subseteq\Y^\X$ is improperly CPAC learnable. 
    Then $\cpsidim(\H)<\infty$, i.e., $\H$ admits a computable $k$-witness of $\Psi$-dimension for some $k\in\N$.
\end{corollary}

The proof of Theorem~\ref{thm:lb-c-nat-c-dim} is based  on our lemma below, which is a computable version of the generalization of the Sauer Lemma to finite label multiclass settings from \citep{natarajan1989learning}.

\begin{lemma}\label{lemma:natarajansauer}
    Let $|\Y| < \infty$ and suppose  $\H\subseteq\Y^\X$ satisfies $\cNdim(\H)= k_N$.
    Then there is a computable function $v$ that takes as input a set $T\in \X^m$ and outputs a set of labellings $v(T) \subseteq \Y^m$ with 
     $v(T) \supseteq \H|_{T}$  and 
       $|v(T)| \leq m^{k_N+1}|\Y|^{2(k_N+1)}.$
\end{lemma}

\begin{proof}
    From Lemma~\ref{lemma:embedding} we know that there is an embedding $\H' \supseteq \H$ such that $\mathrm{Ndim}(\H') = k_N+1$ and such that there exists a computable function $v$, such that for any input $T\in \X^m$ it outputs $\H'|_T$. We can now invoke a classical result from \citep{natarajan1989learning}, which shows that the number of behaviours of any class $\H'$ with Natarajan dimension $d_N$ with finite label space $\Y$ on any set $T\in \X^m$ is upper bounded by $m^{k_N}|\Y|^{2k_N}$.
    Lastly, we note that since $\H \subset \H'$ we have $\H|_T \subseteq v(T)$ and  $|\H|_T| \leq |\H'|_T| = |v(T)| \leq m^{k_N+1}|\Y|^{2(k_N+1)}$, thus proving the bound of the lemma.
\end{proof}

We now proceed with the proof of Theorem~\ref{thm:lb-c-nat-c-dim}.

\begin{proof}[Proof of Theorem~\ref{thm:lb-c-nat-c-dim}]
    Let $\cNdim(\H) =k_N$. Now let $k_B$ be some arbitrary number satisfying the inequality $k_B^{k_N + 1} (|\Y|)^{2(k_N +1)}  < 2^{k_B}$. That is, $k_B$ is an arbitrary number that exceeds the bound for $\cpsidim$ in the theorem.
    We will now prove that bound for $\cpsidim(\H)$, by showing that for such $k_B$ there exists a computable $k_B$-witness function $w_B$ for the $\Psi$-dimension. 
  Let $\bar{\psi}\in \Psi$ and $T\in \X^{k_B}$ be arbitrary.
   Now from Lemma~\ref{lemma:natarajansauer} we know there is a computable function $v$ that for input $T$  
    outputs a set $v(T)\supseteq \H|_T$ with  $|v(T)| \leq k_B^{k_N}(|\Y|+1)^{2k_N} < 2^{k_B}$. This also implies that $|\Psi(v(T))|\leq |v(T)| < 2^{k_B} $.
     In particular, this implies that there is at least one $\{0,1\}$-labelling $g'\not\in \bar{\psi}(v(T))$. Furthermore, we can computably identify $g'$ by checking which labelling is missing from the computably generated set $\bar{\psi}(v(T))$.
     Furthermore, from $v(T) \supseteq \H|_T$, it also follows that $\Psi(v(T)) \supseteq \Psi(\H|_T)$. 
    Thus $g'\notin \Psi(\H|_T)$, i.e. $g'$ is a witness for the set $T$ with distinguisher $\bar{\psi}$. Thus we have shown that $\cpsidim(\H)$ is at most $k_B-1$, implying the bound ot the theorem.
       \end{proof}

\subsection{Sufficient Conditions for Finite $\Y$}

In this section, we show that, for a distinguisher $\Psi$, the finiteness of the computable $\Psi$ dimension, $\cpsidim$, provides a sufficient condition for CPAC learnability.

\begin{theorem}
\label{thm:finite-y+psi-uc=>cpac}
    Let $|\Y|<\infty$.
    Let $\Psi$ be a family of functions from $\Y$ to $\{0,1,*\}$.
    Furthermore suppose that finite $\Psi$-dimension implies uniform convergence under the 0-1 loss.
    Then $\cpsidim(\H)<\infty$ implies that $\H$ is CPAC learnable.
\end{theorem}

The proof follows from suitable generalizations of the arguments presented in Section~\ref{sec:c-ndim-ub-finite}: 

\begin{proof}
    Let $\H\subseteq\Y^\X$ be such that $\cpsidim(\H)<\infty$.
    Let $w:\X^{k+1}\times\Psi^{k+1}\rightarrow\{0,1\}^{k+1}$ be a $k$-witness of $\Psi$-dimension.
    Our goal is to embed $\H$ into $\H'$ satisfying the following: (i) $\psidim(\H')\leq k+1$ and (ii)~$\H'$ has computable ERM. 
    By the conditions of the theorem statement, this is sufficient to guarantee multiclass CPAC learnability.

    \underline{Constructing $\H'$.} 
    Consider the class $\G\subseteq\Y^\N$ of ``good'' functions satisfying, for all $g\in\G$
    \begin{enumerate}
        \item $ M(g)<\infty$, where $M(g)=\arg\max_{n\in\N}\{ g(n)\neq0 \}$ and $M(h)=\infty$ if no such $n$ exists.
        \item For any $x_1<\dots<x_k<x_{k+1}\leq M(g)$, and any $\bar{\psi}\in\Psi^{k+1}$, $w(X,\bar{\psi})\neq\bar{\psi}(g|_X)$, where $X=\{x_i\}_{i\in[k+1]}$.
    \end{enumerate}
    Now, let $\H':=\H\cup\G$. We will show that $\H'$ indeed satisfies the conditions above.

    \underline{$\psidim(\H')\leq k+1$.}
    Let $X=\{x_1,\dots,x_{k+2}\}\in\X^{k+2}$ and $\bar{\psi}\in\Psi^{k+2}$. 
    WLOG suppose $x_1<\dots<x_{k+1}<x_{k+2}$ and that $\bar{\psi}_{k+2}^{-1}(0)$ and $\bar{\psi}_{k+2}^{-1}(1)$ are both non empty. 
    Let $y_0\in \bar{\psi}_{k+2}^{-1}(0)$, $y_1\in \bar{\psi}_{k+2}^{-1}(1)$ be the minimal $y_i$ in their respective set, and WLOG let $y_0<y_1$, in particular $y_1 >0$.
    Consider $w(X_{-(k+2)},\bar{\psi}_{-(k+2)})$, the output of the $k$-witness on $X$ and $\bar{\psi}$, but disregarding the $(k+2)$-th entry.
    Let $w'=(w(X_{-(k+2)},\bar{\psi}_{-(k+2)}),1)\in\{0,1\}^{k+2}$.
    We claim that no $h\in\H'$ satisfies $\bar{\psi}(h|_X)=w'$.
    First note that, by definition, no $h\in\H$ can satisfy this.
    So it remains to consider some ``good'' function $g\in\G$.
    We distinguish two cases:
    \begin{enumerate}
        \item $g(x_{k+2})\neq 0$: then $x_{k+2}\leq M(g)$, but since $g$ is ``good'', $w(X_{-(k+2)},\bar{\psi}_{-(k+2)})\neq\bar{\psi}(g|_{X_{k+2}})$ by definition,
        \item $g(x_{k+2})= 0$: then by construction $0\notin\bar{\psi}_{k+2}^{-1}(1)$, thus  $\bar{\psi}_{k+2}(g(x_{k+2}))\neq 1$, as required.
    \end{enumerate}

    \underline{$\H'$ has a computable ERM.}
    Let $S\in(\N\times\Y)^m$ be arbitrary.
    Let $M=\max_{(x,y)\in S}x$, and note that it suffices to consider functions $h$ with $ M(h)\leq M$, and that by the finiteness of $\Y$, there are a finite number of ``good'' functions in $\G$ satisfying this, and that these functions can be identified computably, by listing all patterns and using the computable witness function to exclude functions from $\G$, with a similar argument as in the proof of Lemma~\ref{lemma:embedding}.
    If we can show that there always exists a function in $\G$ that is an empirical risk minimizer, then we are done.
    Let $h^*\in\arg\min_{h\in\H'}$, and consider its ``truncated'' version $h^*_M$, where if $x\leq M$, $h^*_M(x)=h^*(x)$ and otherwise, $h^*_M(x)=0$.
    $\widehat{\R}_D(h^*)=\widehat{\R}_D(h^*_M)$, so it remains to show $h^*_M\in\G$. 
    Suppose not.
    Then there must exist $X=\{x_1,\dots,x_{k+1}\}\subseteq[M]$ and $\bar{\psi}\in\Psi^{k+1}$ such that $w(X,\bar{\psi})=\bar{\psi}(h^*_M|_X)$, but that means that the non-truncated $h^*\in\H$ also satsifies this, a contradiction by the definition of $w$.
    
\end{proof}

As a corollary of Theorem~\ref{thm:bd92}, we get:

\begin{corollary}
    Let $|\Y|<\infty$.
    Let $\Psi$ be a family of function from $\Y$ to $\{0,1,*\}$.
    Furthermore suppose that $\Psi$ is a distinguisher.
    Then $\cpsidim(\H)<\infty$ implies that $\H$ is CPAC learnable.    
\end{corollary}

Combining these with Corollary~\ref{cor:c-psi-dim-necessary}, we obtain the following result:

\begin{theorem}
    \label{thm:cpsidim<infty<->cpac}
  Let $|\Y|<\infty$ and suppose $\Psi$ is a distinguisher. 
    Then $\cpsidim(\H)<\infty$ if and only if $\H$ is CPAC learnable.
\end{theorem}

\nextsubmission{
\begin{remark}
    Furthermore, we note that the sample complexity of learning a class $\H$ for any distinguisher $\Psi$ can be bounded by 
    $O(\frac{1}{\epsilon}(d |\Psi|(\log(d)\log(|\Psi| + \log(|\Y|)))\ln(\frac{1}{\epsilon}) +\ln(\frac{1}{\delta})) $
    where $d$ is the computable uniform $\Psi$-dimension.
\end{remark}
}
 
\begin{remark}
    Since the Natarajan and Graph dimensions are both expressible as distinguishers (with the computable versions matching the corresponding computable $\Psi$-dimension), Theorem~\ref{thm:cpsidim<infty<->cpac} hold for $\cNdim$ and $\cGdim$.
    Similarly, the result can be obtained for other  families of distinguishers such as the Pollard pseudo-dimension \citep{pollard1990empirical,haussler1992decision}.
\end{remark}

Finally, we note that, while the arguments of Section~\ref{sec:c-ndim-lb} (which give a necessary condition on CPAC learnability via the finiteness of the computable Natarajan dimension) hold for infinite label spaces, neither arguments in Theorem~\ref{thm:lb-c-nat-c-dim}, nor in Theorem~\ref{thm:finite-y+psi-uc=>cpac} can be extended to infinite $\Y$: in the former, $|\Y| $ appears in the denominator of the lower bound; the latter relies on the finiteness of $\Y$ to implement a computable ERM.

\subsection{A Meta-Characterization for CPAC Learnability}

In the previous section, we showed that distinguishers give rise to computable dimensions that qualitatively characterize CPAC learnability for finite $\Y$ in the agnostic setting.
But what happens if a family of functions fails to be a distinguisher?

\begin{proposition}
\label{prop:not-distinguisher}
    Suppose $\Psi$ fails to be a distinguisher. 
    Then there exists $\H$ with $\cpsidim(\H)=1$ such that $\H$ is not CPAC learnable.
\end{proposition}

\begin{proof}
    Suppose it is the case that $\Psi$ fails to be a distinguisher, and say it cannot distinguish labels $y_1,y_2$. 
    Then,  letting $\mathcal{H}$ be the hypothesis class of all functions from $\mathbb{N}$ to $y_1,y_2$, we note that $\mathcal{H}$ is not PAC learnable, and thus not CPAC learnable. 
    Now, to see that $\cpsidim(\mathcal{H})=1$, note that for arbitrary $x\in\mathcal{X}$ and any $\psi\in\Psi$, there is $b\in\{0,1\}$ such that  $\psi(\H|_{\{x\}})\subseteq\{b,*\}$, and $b$ can be identified by computing $\psi(y_1), \psi(y_2)$, regardless of $x$. 
    Thus for any given $x$ and $\psi$, the witness function returns $b\oplus1$, as required. 
\end{proof}

Combining Proposition~\ref{prop:not-distinguisher} with Theorem~\ref{thm:cpsidim<infty<->cpac}, we conclude the main result of this paper: a meta-characterization for CPAC learnability in the agnostic setting, in the sense that we precisely characterize which families of functions from  $\Y$ to $\{0,1,*\}$ give rise to computable dimensions characterizing multiclass CPAC learnability.

\begin{theorem}
    \label{thm:meta-characterization}
    Let $\Psi$ be a family of functions from $\Y$ to $\{0,1,*\}$ for finite $\Y$.
    Then $\cpsidim(\H)$ qualitatively characterizes CPAC learnability if and only if $\Psi$ is a distinguisher.
\end{theorem}

\subsection{The DS Dimension}

The bounds derived in Section~\ref{sec:gen-method} hold for families $\Psi$ of functions from $\Y$ to $\{0,1,*\}$ that satisfy certain properties, e.g., are distinguishers.
This generalization of the Natarajan and Graph dimensions predates the work of \cite{daniely2014optimal}, which defined the DS dimension, a characterization of learnability for multiclass classification for arbitrary label spaces $\Y$.
For infinite label space, \cite{brukhim2022characterization} exhibit an arbritrary gap between the Natarajan and DS dimensions.
But even in the case of a finite label set, can we express the DS dimension as a family $\Psi_{\DS}$, in the sense that for all $\H\subseteq\Y^\X$, $\DS(\H)=\Psi_{\DS}\text{-}\dim(\H)$?
Unfortunately, the result below gives a negative answer to this question, which may be of independent interest.

\begin{lemma}
    The DS dimension cannot be expressed as a family $\Psi_{\DS}$ of functions from $\Y$ to $\{0,1,*\}$.
\end{lemma}

\begin{proof}
     We will show that any family $\Psi$ with $\psidim(\H)=2$ and $\psidim(\H')=1$ induces $\H_*\subset\H$ with $\psidim(\H_*)=2$.

     Let $\Psi$ be such that $\psidim(\H)=2$ and $\psidim(\H')=1$.
     WLOG, suppose there exists $\psi_*\in\Psi$ such that $\{0,1\}\subseteq\psi_*(\H'|_{\{0\}})$.
      Since $\H'|_{\{1\}}=\{2,4\}$, it must be that for all $\psi\in\Psi$, we cannot have both $\psi(2)\neq\psi(4)$ and $\psi(2),\psi(4)\in\{0,1\}$. 
     Let $(\psi_1,\psi_2)\in\Psi^2$ witness the $\Psi$-shattering of $\H$ on $\X$, i.e., 
     \begin{equation}
         \label{eqn:psi-shattering-H}
         \{0,1\}^2\subseteq (\psi_1,\psi_2)(\H)\enspace.
     \end{equation}
     We distinguish three cases:
     \begin{enumerate}
         \item $\psi_2(2)=\psi_2(4)=*$: impossible by Equation~\ref{eqn:psi-shattering-H},
        \item $\psi_2(2)\in\{0,1\}$ and $\psi_2(4)=*$: WLOG let $\psi_2(2)=0$. 
         Then, for Equation~\ref{eqn:psi-shattering-H} to hold, we need $\psi_2(6)=1$ as well as $\psi_1(3)=\psi_1(5)\neq\psi_1(1)$, all of which must be in $\{0,1\}$.
         From this, it is clear that $\H_*=\{12,32,56,16\}$ satisfies $\psidim(\H_*)=2$, despite $\DS(\H_*) =1$. Thus we get an impossibility.
         Note that the cases (a) $\psi_2(2)=*$ and $\psi_2(4)\in\{0,1\}$, (b) $\psi_2(4)\in\{0,1\}$ and $\psi_2(2)=*$ and (c) $\psi_2(4)=*$ and $\psi_2(2)\in\{0,1\}$) follow an identical reasoning.
         \item $\psi_2(2)=\psi_2(4)\in\{0,1\}$: WLOG let $\psi_2(2)=\psi_2(4)=0$. Then, for Equation~\ref{eqn:psi-shattering-H} to hold, we need $\psi_2(6)=1$ as well as $\psi_1(1)\neq\psi_1(5)$, both in $\{0,1\}$. 
         But this implies that the class $\H_*=\{12,16,56,54\}$ satisfies $\psidim(\H_*)=2$,
     \end{enumerate}
     as required.
\end{proof}

\section{Conclusion}

We initiated the study of multiclass CPAC learnability, focusing on finite label spaces, and have established a meta-characterization through the finiteness of the computable dimension of a vast family of functions: so-called distinguishers. Characterization through the computable Natarajan and the computable graph dimensions appear as special cases of this result. 
Moreover, we showed that this result cannot readily be extended to the DS dimension, thus suggesting that characterizing CPAC learnability for infinite label spaces will potentially require significantly different techniques.

\section*{Acknowledgements}
{Pascale Gourdeau has been supported by a Vector Postdoctoral Fellowship and an NSERC Postdoctoral Fellowship. Tosca Lechner has been supported by a Vector Postdoctoral Fellowship. Ruth Urner is also an Affiliate Faculty Member at Toronto's Vector Institute, and acknowldeges funding through an NSERC Discovery grant.}

\bibliographystyle{plainnat}
\bibliography{refs}

\appendix

\section{Proof of Proposition~\ref{prop:cndim-cgdim-gap}}
\label{appx:cndim-cgdim-gap}

\begin{proof}[of Proposition~\ref{prop:cndim-cgdim-gap}]
    Let $\X$ be finite or countable. 
    Consider the hypothesis class in \cite{daniely2015multiclass} showing the same separation between the Natarajan and graph dimensions: let $\mathcal{P}_\mathrm{f}(\X)$ be a subset of the powerset of $\X$ consisting only of finite or cofinite subsets of $\X$. 
    Let $\Y=\mathcal{P}_\mathrm{f}(\X)\cup\{\star\}$.
    For any $A\in \mathcal{P}_\mathrm{f}(\X)$, let 
    $$h_A(x)=\begin{cases}
        A & x\in A\\
        \star & x\notin A
    \end{cases}\enspace.$$
    Note that any $A\in\mathcal{P}_\mathrm{f}(\X)$ has a finite representation (a special character for whether we are enumerating the set or its complement, as well as the finite set $A$ or $\X\setminus A$).
    Thus checking whether $x\in A$ can be done computably for any $x\in\X$, implying each $h_A$ is computably evaluable.  
    Since $\Gdim(\H)=|\X|$, it follows that $\cGdim(\H)=|\X|$ as well.
    We will now show that $\cNdim(\H)=1$, namely we exhibit a \emph{computable} $1$-witness of Natarajan dimension. 
    To this end, let $X=\{x_1,x_2\}\in\X^2$ and $y,y'\in\Y^2$ with $y_1\neq y _1'$ and $y_2\neq y_2'$ be arbitrary.
    First check whether $\{y_1, y_2, y_1', y_2' \}$ contains more than one non-$\star$ label, in which case we are done, as we can output the index set corresponding to  labelling $AB$ for some $A,B\in\mathcal{P}_\mathrm{f}(\X)$.
    Otherwise, WLOG let $y=AA$ and $y'=\star\star$ for some $A\in\mathcal{P}_\mathrm{f}(\X)$.
    Check whether $x_1\in A$.
    If yes, output $I=2$ (corresponding to labelling $\star A$), and if not output $I=1$ (corresponding to labelling $A\star$).
\end{proof}

\end{document}